\documentclass{article}
\usepackage{amsthm,amsmath,bbm,amsfonts,mathrsfs,eucal,algorithmic,algorithm,url,graphicx,subfigure,color}

\usepackage{textcomp}
\usepackage{multirow} 

\definecolor{dg}{cmyk}{1,0,1,.5}

\newcommand{\R}{ \mbox{${\mathbb R}$}}

\newcommand{\argmax}{\operatornamewithlimits{argmax}}

\newcommand{\ls}{ \left[}
\newcommand{\rs}{ \right]}
\newcommand{\tth}{ \text{th}}

\newtheorem{lemma}{Lemma}
\newtheorem{theorem}{Theorem}

\title{An upper bound on prototype set size for condensed nearest neighbor}
\author{Eric Christiansen}

\begin{document}

\maketitle


\begin{abstract}
The condensed nearest neighbor (CNN) algorithm is a heuristic for reducing the number of prototypical points stored by a nearest neighbor classifier, while keeping the classification rule given by the reduced prototypical set consistent with the full set.
I present an upper bound on the number of prototypical points accumulated by CNN. The bound originates in a bound on the number of times the decision rule is updated during training in the multiclass perceptron algorithm, and thus is independent of training set size.
\end{abstract}

\section{Introduction}
The nearest neighbor (NN) rule assigns to an unclassified point the class of a closest point from a set of prototypical points. The NN algorithm stores every training point as a prototypical point and classifies new points according to the NN rule. 
A nice property is that, for arbitrary class distributions, as the number of training points goes to infinity, the error of the rule produced by the NN algorithm converges to within twice the Bayes error \cite{Devroye}. 
Unfortunately, storing every training point as a prototypical point can be impractical for huge training sets in terms of both memory complexity and the time complexity of classifying according to the NN rule. As a result, many techniques exist for reducing the size of the set of prototypical points. See \cite{Wilson} and \cite{Toussaint} for an overview. It is suggested in \cite{Angiulli} that using a smaller set with the NN rule may be preferable to using the entire training set, because the VC dimensionality of a NN classifier is given by the number of its prototypical points. However, experimental results in \cite{Wilson} suggest generalization performance is best when the entire training set is used, illustrating a tradeoff between computational complexity and generalization performance. 

This paper focuses on an algorithm for finding a consistent subset of the training data, where a set of prototypes is consistent if it classifies the training set correctly using the NN rule. Finding a consistent subset of minimum cardinality is intractable \cite{Wilfong}, but several heuristic approaches exist. One approach is the condensed nearest neighbor (CNN) algorithm, a simple method introduced by Hart \cite{HartCNN} that has met with empirical success \cite{Angiulli}\cite{Wilson}. In this paper, I point out a striking similarity between the CNN algorithm and the multiclass perceptron algorithm described by Collins in \cite{CollinsPerceptron}. I develop that connection to derive an upper bound on the number of prototypical points accumulated by CNN. 
The existence of this bound may help explain CNN's success in \cite{Angiulli} and \cite{Wilson}.

\section{Condensed nearest neighbor}
We will first need some definitions. Let $T = \{ ({\bf x}_i, c_i) \}$ be a set of classified points, where each point ${\bf x}_i \in \R^d$ and each label $c_i$ comes from a finite set of classes $C$.
Let $N_T({\bf x}) \in \R^d$ be the point part of the element in $T$ minimizing Euclidean distance to ${\bf x}$; so $N_T({\bf x})$ is ${\bf x}$'s nearest neighbor in $T$.\footnote{To simplify analysis, we shall assume $N_T$ is always well defined, meaning that every point we consider has exactly one nearest neighbor in $T$. We will also assume there do not exist $({\bf x},c),({\bf x}',c') \in T$ such that ${\bf x} = {\bf x}'$ and $c \neq c'$; in other words, there are no conflicting labels.} Let $C_T({\bf x}) \in C$ be the class of $N_T({\bf x})$. If $T = \emptyset$, then neither $N_T$ nor $C_T$ exist, and it is understood they cannot be said to equal anything. A set $P$ is said to be \textbf{consistent} with $T$ if $C_P({\bf x}) = c$ for every $({\bf x},c) \in T$.

Suppose $T = \{({\bf x}_i,c_i )\}$ is our training set. Then the condensed nearest neighbor (CNN) algorithm builds a consistent subset of prototypical points $P$ \cite{HartCNN}. Its definition is given in Algorithm \ref{alg:CNN}. Here the \texttt{for} loop selects points in arbitrary order, but without repetition, from $T$.

\begin{algorithm}
\caption{Condensed nearest neighbor algorithm}
\label{alg:CNN}
\begin{algorithmic}
\STATE $P \leftarrow \emptyset$
\STATE flag $\leftarrow$ True
\WHILE {flag is True}
	\STATE flag $\leftarrow$ False
	\FOR {$({\bf x},c)$ in $T$}
		\IF {$C_P({\bf x}) \neq c$}
			\STATE $P \leftarrow P \cup \{({\bf x} , c)\}$
			\STATE flag $\leftarrow$ True
		\ENDIF
	\ENDFOR
\ENDWHILE
\end{algorithmic}
\end{algorithm}

For finite $T$, this algorithm is guaranteed to terminate; in the worst case, it stops after adding every element of $T$ to $P$. It is also easy to see that after termination, $P$ will be consistent with $T$. In practice, CNN often finds a prototypical set far smaller than the original training set \cite{Angiulli}\cite{Wilson}, though this improvement is obviously more pronounced when the training set is much larger than it needs to be.

CNN can be naturally modified for the online context; the \texttt{while} loop is dropped and it is understood the \texttt{for} loop becomes an iteration over an infinite stream of training data. However, in the case of overlapping class distributions (non-zero Bayes' error), the expected size of $P$ will grow linearly with the number of observed points, perhaps making CNN unsuitable.

\section{Multiclass perceptron}
The multiclass perceptron (MP) algorithm \cite{CollinsPerceptron} is listed as Algorithm \ref{alg:cperc}. It can be thought of as a generalization of the kernelized perceptron algorithm (\cite{bishop-2006}, pages 192-196) to handle multiple classes. Here, $\phi:\R^d \times C \rightarrow H$ is an arbitrary feature function, where $H$ is a Hilbert space. The algorithm builds a decision rule through iterative updates to a vector ${\bf w}$. Note a nonlinear decision rule can be obtained through the use of a nonlinear feature function.

\begin{algorithm}
\caption{Multiclass perceptron algorithm}
\label{alg:cperc}
\begin{algorithmic}
\STATE ${\bf w} \leftarrow {\bf 0}$
\STATE flag $\leftarrow$ True
\WHILE {flag is True}
\STATE flag $\leftarrow$ False
\FOR {$({\bf x},c)$ in $T$}
	\STATE $y \gets \argmax_{y \in C} {\bf w} \cdot \phi({\bf x},y)$
	\IF {$y \neq c$}
		\STATE ${\bf w} \leftarrow {\bf w} + \phi({\bf x},c)-\phi({\bf x},y)$
		\STATE flag $\leftarrow$ True
	\ENDIF
\ENDFOR
\ENDWHILE
\end{algorithmic}
\end{algorithm}

I will say a set $T=\{({\bf x}_i,c_i)\}$ is \textbf{separable with margin $\delta > 0$} if there exists a unit vector ${\bf w}^* \in H$ such that for every $({\bf x},c) \in T$ and for every $y \in C-\{c\}$ we have ${\bf w}^* \cdot \phi({\bf x},c) - {\bf w}^* \cdot \phi({\bf x},y) \geq \delta$. We will say $T$ has \textbf{radius at most $R$} if for every $({\bf x},c) \in T$ and for every $y \in C-\{c\}$ we have $||\phi({\bf x},c) - \phi({\bf x},y)|| \leq R$. Note both margin and radius depend on the feature function $\phi$ as well as the data $T$. From \cite{CollinsPerceptron} we have Theorem \ref{thm:updates}.

\begin{theorem}
\label{thm:updates}
Let $T$ be separable with margin $\delta$ and radius at most $R$. Then the MP algorithm updates ${\bf w}$ at most $R^2/\delta^2$ times.
\end{theorem}

Note the number of updates to ${\bf w}$ is the same as the number of points $({\bf x},c) \in T$ that are misclassified during the construction of the decision rule. Also note if $T$ is not separable, then it will not have a margin $\delta > 0$. In fact, the MP algorithm will never terminate.

\section{An upper bound on prototype set size for condensed nearest neighbor}
Here I develop the connection between the condensed nearest neighbor (CNN) and multiclass perceptron (MP) algorithms. The motivating insight comes from support vector machine (SVM) lore, which asserts that an SVM with a Gaussian kernel is like a smoothed nearest neighbor (NN) classifier. In a similar vein, I will show that for certain $\phi$s and ${\bf w}$s, we can use ${\bf w} \cdot \phi({\bf x},y)$ as a proxy for the nearness of ${\bf x}$ and its nearest neighbor of class $y$, making $\argmax_{y \in C} {\bf w} \cdot \phi({\bf x},y)$ the class of ${\bf x}'s$ nearest neighbor.

First we need some definitions. Let $T = \{ ({\bf x}_i, c_i) \}$ be a training set, and let $P \subseteq T$. Let $o:\R^d \rightarrow C$ be any multifunction, where for every $({\bf x},c) \in T$, we have $o({\bf x}) \neq c$. 
So $o({\bf x})$ can be any class, except the true class of ${\bf x}$.
I say the set of all ${\bf w}$s of the form given in (\ref{eqn:restricted}) are the \textbf{restricted} ${\bf w}$s corresponding to $P$.
\begin{align}
{\bf w} 
&= \sum_{({\bf x},c) \in P} \phi({\bf x},c) - \phi({\bf x},o({\bf x}))\label{eqn:restricted}
\end{align}
Note in (\ref{eqn:restricted}), ${\bf w}$ has exactly the form one might expect the ${\bf w}$ from the MP algorithm to have after termination, except that every summand in a restricted ${\bf w}$ has the unitary coefficient.

Now I define a condition on $\phi$ that makes $\argmax_{y \in C} {\bf w} \cdot \phi({\bf x},y)$ behave like the NN rule. I will say $\phi$ is \textbf{neighborly} (with respect to $T$) if for every $P \subseteq T$ and every corresponding restricted ${\bf w}$, as well as for every training point $({\bf x},c) \in T$, we have $\argmax_{y \in C} {\bf w} \cdot \phi({\bf x},y) = C_P({\bf x})$. When $|T|$ is finite, we can always find uncountably many neighborly $\phi$s; see Appendix \ref{sec:appdx} for a construction using the Gaussian kernel.

By Lemma \ref{thm:restricted}, when $\phi$ is neighborly, the MP algorithm will never misclassify the same point twice. So we will only need to consider restricted ${\bf w}$s when analyzing the MP algorithm. 

\begin{lemma}
\label{thm:restricted}
Let $\phi$ be neighborly. Then at all times during the execution of Algorithm \ref{alg:cperc}, ${\bf w}$ will be restricted.
\end{lemma}
\begin{proof}
This is a proof by induction. Note ${\bf w}$ is initially restricted. 
Further, if ${\bf w}$ is restricted before the $n^\tth$ update, then it will be restricted after the $n^\tth$ update. 
For let $({\bf x},c) \in T$ be the point causing the $n^\tth$ update. So $\argmax_{y \in C} {\bf w} \cdot \phi({\bf x},y) \neq c$. But $\phi$ is neighborly, and ${\bf w}$ is restricted before this update, so $\argmax_{y \in C} {\bf w} \cdot \phi({\bf x},y) = C_P({\bf x})$. So $C_P({\bf x}) \neq c$. So $({\bf x},c) \not \in P$, and thus ${\bf w}$ is restricted after the $n^\tth$ update.
\end{proof}

The CNN algorithm is essentially the MP algorithm \ref{alg:cperc} in the special case $\phi$ is neighborly, so we should be able to apply to the CNN algorithm theorems that apply to the MP algorithm. In that spirit, Lemma \ref{thm:mainlemma} says we can apply Theorem \ref{thm:updates} to the CNN algorithm.

\begin{lemma}
\label{thm:mainlemma}
Let $\phi$ be neighborly, and suppose $T$ is separable with margin $\delta$ and radius at most $R$. Then the CNN algorithm updates $P$ at most $R^2 / \delta^2$ times.
\end{lemma}
\begin{proof}
The idea is to morph the MP algorithm into the CNN algorithm without changing the number of misclassifications, which are in one-to-one correspondence with the number of updates to ${\bf w}$ in the MP algorithm, and are in one-to-one correspondence with the number of updates to $P$ in the CNN algorithm.

We begin by inserting two lines of code into the MP algorithm, obtaining Algorithm \ref{alg:cperc_hyb}. The new lines in Algorithm \ref{alg:cperc_hyb} are lines \ref{alg:cperc_hyb_1} and \ref{alg:cperc_hyb_2}. These lines clearly have no effect on the number of updates to ${\bf w}$, but note in Algorithm \ref{alg:cperc_hyb} the number of updates to $P$ equals the number of updates to ${\bf w}$.
\begin{algorithm}[h!]
\caption{Hybrid multiclass perceptron algorithm}
\label{alg:cperc_hyb}
\begin{algorithmic}[1]
\STATE ${\bf w} \leftarrow {\bf 0}$
\STATE $P \leftarrow \emptyset$ \label{alg:cperc_hyb_1}
\STATE flag $\leftarrow$ True
\WHILE {flag is True}
\STATE flag $\leftarrow$ False
\FOR {$({\bf x}_i,c_i)$ in $T$}
	\STATE $y_i \gets \argmax_{y \in C} {\bf w} \cdot \phi({\bf x}_i,y)$
	\IF {$y_i \neq c_i$}
		\STATE ${\bf w} \leftarrow {\bf w} + \phi({\bf x}_i,c_i)-\phi({\bf x}_i,y_i)$
		\STATE $P \leftarrow P \cup \{({\bf x}_i , c_i)\}$ \label{alg:cperc_hyb_2}
		\STATE flag $\leftarrow$ True
	\ENDIF
\ENDFOR
\ENDWHILE
\end{algorithmic}
\end{algorithm}
We then use Lemma \ref{thm:restricted}, together with the assumption $\phi$ is neighborly, to replace $\argmax_{y \in C} {\bf w} \cdot \phi({\bf x}_i,y)$ with $C_P({\bf x}_i)$ in Algorithm \ref{alg:cperc_hyb}, yielding Algorithm \ref{alg:CNN_hyb}.
\begin{algorithm}[h!]
\caption{Hybrid condensed nearest neighbor algorithm}
\label{alg:CNN_hyb}
\begin{algorithmic}[1]
\STATE ${\bf w} \leftarrow {\bf 0}$ \label{alg:CNN_hyb_1}
\STATE $P \leftarrow \emptyset$
\STATE flag $\leftarrow$ True
\WHILE {flag is True}
\STATE flag $\leftarrow$ False
\FOR {$({\bf x}_i,c_i)$ in $T$}
	\IF {$C_P({\bf x}_i) \neq c_i$}
		\STATE ${\bf w} \leftarrow {\bf w} + \phi({\bf x}_i,c_i)-\phi({\bf x}_i,C_P({\bf x}_i))$ \label{alg:CNN_hyb_2}
		\STATE $P \leftarrow P \cup \{({\bf x}_i , c_i)\}$
		\STATE flag $\leftarrow$ True
	\ENDIF
\ENDFOR
\ENDWHILE
\end{algorithmic}
\end{algorithm}
Note Algorithm \ref{alg:CNN_hyb} updates $P$ exactly as many times as the MP algorithm updated ${\bf w}$, and from Algorithm \ref{alg:CNN_hyb} we can simply delete both lines referencing ${\bf w}$ (lines \ref{alg:CNN_hyb_1} and \ref{alg:CNN_hyb_2}). This yields the CNN algorithm, without changing the number of updates to $P$. So by Theorem \ref{thm:updates}, we are done.
\end{proof}

Since Lemma \ref{thm:mainlemma} gives a bound for any neighborly feature function, the set of all neighborly feature functions yields a set of bounds. It is natural to choose the best bound, which is the essence of Theorem \ref{thm:maintheorem}. Note the number of updates to $P$ is exactly the ultimate size of $P$, or equivalently the number of points accumulated by the CNN algorithm.

\begin{theorem}
\label{thm:maintheorem}
Let $\Phi = \{ \phi : \phi \text{ is neighborly} \}$ and suppose, for each $\phi \in \Phi$, $T$ is separable with margin $\delta_\phi$ and radius at most $R_\phi$. Then the CNN algorithm accumulates at most $\inf_{\phi \in \Phi} R_\phi^2 / \delta_\phi^2$ prototypical points.
\end{theorem}
\begin{proof}
By Lemma \ref{thm:mainlemma}, for any neighborly $\phi$, the CNN algorithm accumulates at most $R_\phi^2/\delta_\phi^2$ representative points. Thus since $\Phi$ is the set of all neighborly $\phi$s, the CNN algorithm accumulates at most $\inf_{\phi \in \Phi} R_\phi^2/\delta_\phi^2$ representative points.
\end{proof}

So $|P|$ is bounded above by the best bound over all neighborly $\phi$s. 

\section{Conclusion}
In Theorem \ref{thm:maintheorem} I presented a bound on the number of prototypical points accumulated by the condensed nearest neighbor (CNN) algorithm. This bound came from a bound on the number of updates to the decision rule in the multiclass perceptron algorithm. Unfortunately, as with the multiclass perceptron bound presented in \cite{CollinsPerceptron}, estimating the bound is likely to be too expensive to be practical. Fortunately, as with the multiclass perceptron bound, this bound is independent of the size of the training set used as input to the CNN algorithm. Perhaps the existence of this bound may help explain CNN's empirical success.

\appendix

\section{Neighborly $\phi$s}
\label{sec:appdx}
Let $T = \{ ({\bf x}_i,c_i) \}$ be a finite training set, where ${\bf x}_i \in \R^d$, $c_i \in C$, and $|C|$ is finite. I will show there exist uncountably many neighborly feature functions $\phi : \R^d \times C \rightarrow H$, where $H$ is a Hilbert space. In other words, there exist uncountably many functions $\phi : \R^d \times C \rightarrow H$ such that for any $P \subseteq T$ and restricted ${\bf w}$ corresponding to $P$, as well as for any $({\bf x},c) \in T$, we have $\argmax_{y \in C} {\bf w} \cdot \phi({\bf x},y) = C_P({\bf x})$. The idea is to construct an uncountable family of such functions using the Gaussian kernel.

Let $\psi_\sigma:\R^d \rightarrow H$ correspond to the Gaussian kernel, so that 
\[
\psi_\sigma({\bf x}) \cdot \psi_\sigma({\bf x}') = e^{\frac{-1}{2 \sigma^2} ||{\bf x} - {\bf x}'||^2}.
\] 
We know such a $\psi_\sigma$ exists (\cite{bishop-2006}, pages 294-299). Let $\phi_\sigma({\bf x},c) \in H$ be a vector such that the $i^{\text{th}}$ element $\phi_\sigma({\bf x},c)_i$ of $\phi_\sigma({\bf x},c)$ is
\[
\phi_\sigma({\bf x},c)_i = 
\left\{ \begin{tabular}{ c l}
  $\psi_\sigma({\bf x})_j$ & if there exists a $j$ such that $i = (j-1)|C| + c$  \\
  $0$ & otherwise  \\
\end{tabular}\right..
\]
The idea behind the above definition is to make
\[
\phi_\sigma({\bf x},c)  \cdot \phi_\sigma({\bf x}',c') = I(c = c') \psi_\sigma({\bf x}) \cdot \psi_\sigma({\bf x}'),
\]
where $I$ is the indicator function on Boolean inputs given by $I(\texttt{true})=1$ and $I({\texttt{false}}) = 0$.

I claim $\phi_\sigma$ will be neighborly if we choose $\sigma > 0$ sufficiently small. First we need the following result.

\begin{lemma}
\label{thm:appdxlemma}
Let $P \subseteq T$ and let ${\bf w}$ be restricted corresponding to $P$. Let $({\bf x}',c') \in T$. Then there exists a $\sigma^*>0$ such that for every $0 < \sigma < \sigma^*$, we have $\argmax_{y \in C} {\bf w} \cdot \phi_\sigma({\bf x}',y) = C_P({\bf x}')$.
\end{lemma}
\begin{proof}
Note ${\bf w}$ is of the form
\begin{align*}
{\bf w} 
&= \sum_{({\bf x},c) \in P} \phi_\sigma({\bf x},c) - \phi_\sigma({\bf x},o({\bf x})).
\end{align*}
So 
\begin{align}
{\bf w} \cdot \phi_\sigma({\bf x}',y)
&= \sum_{({\bf x},c) \in P} I(c=y) \psi_\sigma({\bf x}) \cdot \psi_\sigma({\bf x}') - I(o({\bf x})=y) \psi_\sigma({\bf x}) \cdot \psi_\sigma({\bf x}') \\
&= \sum_{({\bf x},c) \in P} \ls I(c=y) - I(o({\bf x})=y) \rs e^{\frac{-1}{2 \sigma^2} ||{\bf x} - {\bf x'} ||^2}. \label{eqn:appndxdom}
\end{align}
I want to show that as $\sigma \rightarrow 0$, the summation in (\ref{eqn:appndxdom}) will be dominated by the term corresponding to the nearest neighbor of ${\bf x}'$ in $P$. The intuition behind this is that the summands of (\ref{eqn:appndxdom}) correspond to individual Gaussians in a mixture of Gaussians, where the typical constraints on the mixing coefficients are dropped. For large $\sigma$, the value of the mixture at a point might depend significantly on several nearby Gaussians, but as we uniformly shrink the variance of the Gaussians, the value of the mixture at a point will depend almost entirely on the Gaussian nearest that point, as is borne out in the following analysis. First, I divide through by $e^{\frac{-1}{2 \sigma^2} ||N_P({\bf x}') - {\bf x'} ||^2}$, giving us
\begin{align}
(\ref{eqn:appndxdom})
&\propto \sum_{({\bf x},c) \in P} \ls I(c=y) - I(o({\bf x})=y) \rs e^{\frac{1}{2 \sigma^2} \ls ||N_P({\bf x}') - {\bf x'} ||^2 - ||{\bf x} - {\bf x'} ||^2 \rs} \\
&= \ls I(C_P({\bf x}')=y) - I(o(N_P({\bf x}'))=y) \rs \nonumber \\
&\quad + \sum_{({\bf x},c) \in P:{\bf x} \neq N_P({\bf x}')} \ls I(c=y) - I(o({\bf x})=y) \rs e^{\frac{1}{2 \sigma^2} \ls ||N_P({\bf x}') - {\bf x'} ||^2 - ||{\bf x} - {\bf x'} ||^2 \rs}. \label{eqn:appndxsep}
\end{align}
Let $S$ be the summation given in (\ref{eqn:appndxsep}). Note for every $({\bf x},c) \in P$ such that ${\bf x} \neq N_P({\bf x}')$, we have
\[
||N_P({\bf x}') - {\bf x'} ||^2 - ||{\bf x} - {\bf x'} ||^2 < 0.
\]
Thus as $\sigma \rightarrow 0$, every term in S will go to zero. So by the definition of a limit, there is a $\sigma^* > 0$ such that $\sigma < \sigma^*$ implies $|S| < \frac{1}{2}$. Let $\sigma < \sigma^*$. Then if $y = C_P({\bf x}')$, we have $(\ref{eqn:appndxsep}) \geq 1 - |S| > \frac{1}{2}$. If $y \neq C_P({\bf x}')$, we have $(\ref{eqn:appndxsep}) \leq 0 + |S| < \frac{1}{2}$. So $\argmax_{y \in C} {\bf w} \cdot \phi_\sigma({\bf x}',y) = C_P({\bf x}')$.
\end{proof}

With Lemma \ref{thm:appdxlemma} I can prove Theorem \ref{thm:appdxthm}, establishing the desired result.

\begin{theorem}
\label{thm:appdxthm}
For any finite $T$, there exist uncountably many neighborly feature functions.
\end{theorem}
\begin{proof}
By Lemma \ref{thm:appdxlemma}, for every $P \subseteq T$, restricted ${\bf w}$, and $({\bf x}',c) \in T$, we can find a $\sigma^*>0$ such that $\sigma < \sigma^*$ implies $\argmax_{y \in C} {\bf w} \cdot \phi_\sigma({\bf x}',y) = C_P({\bf x}')$. Let $\Sigma^*$ be the set of all such $\sigma^*$s. Since there are finitely many subsets $P$ of $T$, finitely many restricted ${\bf w}$s corresponding to any given $P$, and finitely many elements in $T$, the set $\Sigma^*$ is finite and thus has a minimum element $\sigma^{\text{min}}$. 

Thus for any of the uncountably many $\sigma$s such that $\sigma < \sigma^{\text{min}}$, there exists a neighborly feature function $\phi_\sigma$. Thus the set of all neighborly feature functions contains a subset of uncountable cardinality, and is thus uncountable.
\end{proof}

\section*{Acknowledgements}

I would like to thank Mehran Bozorgi, Charles Elkan, and Matus Telgarsky for acting as sounding boards for this idea, and I would like to thank Charles Elkan and Nicolaus Hepler for constructively criticizing the drafts of this paper.

This work was funded by NSF grant \#SBE-0542013 to the Temporal Dynamics of Learning Center, an NSF Science of Learning Center.

\bibliographystyle{plain}
\bibliography{library.bib}
\end{document}